\documentclass[twoside]{article}
\usepackage{ifpdf}
\pdfoutput=1
\usepackage{fullpage}
\usepackage{amsfonts}
\usepackage{hyperref,url}
\usepackage[round]{natbib}
\usepackage{multirow}
\usepackage{algorithm,algorithmic,amsmath,amsthm}
\usepackage{verbatim,graphicx,subfigure}

\newcommand{\BlackBox}{\rule{1.5ex}{1.5ex}}  

\DeclareMathOperator{\tr}{tr}
 
\newtheorem{thm}{Theorem}
\newtheorem{lemma}{Lemma} 
\newtheorem{proposition}{Proposition}

\def\cDn{\mathcal{D}_{n}}
\def\bbE{\mathbb{E}}
\def\bbP{\mathbb{P}}
\def\bbR{\mathbb{R}}
\def\cD{\mathcal{D}}

\def\cH{\mathcal{H}}
\def\cDn{\mathcal{D}_n}
\def\cX{\mathcal{X}}
\def\cF{\mathcal{F}}
\def\cO{\mathcal{O}}
\def\cP{\mathcal{P}}
\def\cY{\mathcal{Y}}
\def\cS{\mathcal{S}}
\def\Qit{Q_{i}^t}
\def\pit{p_{i}^t}
\def\Qitau{Q_{i}^{\tau}}
\def\pitau{p_{i}^{\tau}}
\def\ha{h_{A}}
\def\Sighz{\hat{\Sigma}_z}
\def\Sighzi{\Sighz^{-1}}
\def\Sigi{\Sigma^{-1}}
\def\Sig{\Sigma}
\def\Sigh{\hat{\Sigma}}
\def\psiz{\psi_z}
\def\hatLth{\hat{L}_{t}(h)}
\def\lambdamax{\lambda_{\text{max}}}
\def\lambdamin{\lambda_{\text{min}}}
\newcommand{\defeq}{\mbox{$\;\stackrel{\mbox{\tiny\rm def}}{=}\;$}}

\begin{document}
\title{UPAL: Unbiased Pool Based Active Learning}
\author{ Ravi Ganti, Alexander Gray\\
School of Computational Science \& Engineering,
Georgia Tech\\
gmravi2003@gatech.edu, agray@cc.gatech.edu}
\maketitle 
\begin{abstract}
  In this paper we address the problem of pool based active learning, and provide an algorithm, called UPAL, that works by minimizing the unbiased estimator of the risk of a hypothesis in a given hypothesis space. For the space of linear classifiers and the squared loss we show that UPAL is equivalent to an exponentially weighted average forecaster. Exploiting some recent results regarding the spectra of random matrices allows us to establish consistency of UPAL when the true hypothesis is a linear hypothesis. Empirical comparison with an active learner implementation in Vowpal Wabbit, and a previously proposed pool based active learner implementation show good empirical performance and better scalability. 
\end{abstract}
\section{Introduction}
In the problem of binary classification one has a distribution $\cD$ on the domain $\cX\times\cY\subseteq \bbR^d\times\{-1,+1\}$, and access to a sampling oracle, which provides us i.i.d. labeled samples $\cS=\{(x_1,y_1),\ldots,(x_n,y_n)\}$. The task is to learn a classifier $h$, which predicts well on unseen points. For certain problems the cost of obtaining labeled samples can be quite expensive. For instance consider the task of speech recognition. Labeling of speech utterances needs trained linguists, and can be a fairly tedious task. Similarly in information extraction, and in natural language processing one needs expert annotators to obtain labeled data, and gathering huge amounts of labeled data is not only tedious for the experts but also expensive. In such cases it is of interest to design learning algorithms, which need only a few labeled examples for training, and also guarantee good performance on unseen data. 

Suppose we are given a labeling oracle $\cO$, which when queried with an unlabeled point $x$ returns the label $y$ of $x$.  Active learning algorithms query this oracle as few times as possible and learn a provably good hypothesis from these labeled samples. Broadly speaking active learning (AL) algorithms can be classified into three kinds, namely membership query (MQ) based algorithms, stream based algorithms and pool based algorithms. All these three kinds of AL algorithms query the oracle $\cO$ for the label of the point, but differ in the nature of the queries. In MQ based algorithms the active learner can query for the label of a point in the input space $\cX$, but this query might not necessarily be from the support of the marginal distribution $\cD_{\cX}$. With human annotators MQ algorithms might work poorly as was demonstrated by Lang and Baum in the case of handwritten digit recognition~\citeyearpar{baum1992query}, where the annotators were faced with the awkward situation of labeling semantically meaningless images. Stream based AL algorithms~\citep{cohn1994improving,chu2011unbiased} sample a point $x$ from the marginal distribution $\cD_{\cX}$, and decide on the fly whether to query $\cO$ for the label of $x$? Stream based AL algorithms tend to be computationally efficient, and most appropriate when the underlying distribution changes with time. Pool based AL algorithms assume that one has access to a large pool $\cP=\{x_1,\ldots,x_n\}$ of unlabeled i.i.d. examples sampled from $\cD_{\cX}$, and given budget constraints $B$, the maximum number of points they are allowed to query, query the most informative set of points. Both pool based AL algorithms, and stream based AL algorithms overcome the problem of awkward queries, which MQ based algorithms face. However in our experiments we discovered that stream based AL algorithms tend to query more points than necessary, and have poorer learning rates when compared to pool based AL algorithms.

\subsection{Contributions.}
\begin{enumerate}
 \item In this paper we propose a pool based active learning algorithm called UPAL, which  given a hypothesis space $\cH$, and a margin based loss function $\phi(\cdot)$ minimizes a provably unbiased estimator of the risk $\bbE[\phi(y h(x))]$. While unbiased estimators of risk have been used in stream based AL algorithms, no such estimators have been introduced for pool based AL algorithms. We do this by using the idea of importance weights introduced for AL in Beygelzimer et al.~\citeyearpar{beygelzimer2009importance}. Roughly speaking UPAL proceeds in rounds and in each round puts a probability distribution over the entire pool, and samples a point from the pool. It then queries for the label of the point. The probability distribution in each round is determined by the current active learner obtained by minimizing the importance weighted risk over $\cH$. Specifically in this paper we shall be concerned with linear hypothesis spaces, i.e. $\cH=\bbR^d$.

\item In theorem~\ref{thm:ewa} (Section~\ref{sec:ewa}) we show that for the squared loss UPAL is equivalent to an exponentially weighted average (EWA) forecaster commonly used in the problem of learning with expert advice~\citep{cesa2006prediction}. Precisely we show that if each hypothesis $h\in\cH$ is considered to be an expert and the importance weighted loss on the currently labeled part of the pool is used as an estimator of the risk of $h\in\cH$, then the hypothesis learned by UPAL is the same as an  EWA forecaster. Hence UPAL can be seen as  pruning the hypothesis space, in a soft manner, by placing a probability distribution that is determined by the importance weighted loss of each classifier on the currently labeled part of the pool.

\item In section~\ref{sec:consistency} we prove consistency of UPAL with the squared loss, when the true underlying hypothesis is a linear hypothesis. Our proof employs some elegant results from random matrix theory regarding eigenvalues of sums of random matrices~\citep{hsu2011analysis,hsu2011dimension,tropp2010user}. While it should be possible to improve the constants and exponent of dimensionality involved in $n_{0,\delta},T_{0,\delta},T_{1,\delta}$ used in theorem~\ref{thm:main}, our results qualitatively provide us the insight that the  the label complexity with the squared loss will depend on the condition number, and the minimum eigenvalue of the covariance matrix $\Sigma$. This kind of insight, to our knowledge, has not been provided before in the literature of active learning.

\item In section~\ref{sec:expts} we provide a thorough empirical analysis of UPAL comparing it to the active learner implementation in Vowpal Wabbit (VW)~\citep{langford2010vowpal}, and a batch mode active learning algorithm, which we shall call as BMAL~\citep{hoi2006batch}. These experiments demonstrate the positive impact of importance weighting, and the better performance of UPAL over the VW implementation. We also empirically demonstrate the scalability of UPAL over BMAL on the MNIST dataset. When we are required to query a large number of points UPAL is upto 7 times faster than BMAL. 
\end{enumerate}
\section{Algorithm Design}\label{sec:alg_design}
A good active learning algorithm needs to take into account the fact that the points it has queried might not reflect the true underlying marginal distribution. This problem is similar to the problem of dataset shift~\citep{quinonero2008dataset} where the train and test distributions are potentially different, and the learner needs to take into account this bias during the learning process. One approach to this problem is to use importance weights, where during the training process instead of weighing all the points equally the algorithm weighs the points differently. UPAL proceeds in rounds, where in each round $t$, we put a probability distribution $\{p_i^t\}_{i=1}^n$ on the entire pool $\cP$, and sample one point from this distribution. If the sampled point was queried in one of the previous rounds $1,\ldots,t-1$ then its queried label from the previous round is reused, else the oracle $\cO$ is queried for the label of the point. Denote by $\Qit\in\{0,1\}$ a random variable that takes the value 1 if the point $x_i$ was queried for it's label in round $t$ and 0 otherwise. In order to guarantee that our estimate of the error rate of a hypothesis $h\in\cH$ is unbiased we use importance weighting, where a point $x_i\in\cP$ in round $t$ gets an importance weight of $\frac{\Qit}{\pit}$. Notice that by definition $\bbE[\Qit|\pit]=1$. We formally prove that importance weighted risk is an unbiased estimator of the true risk. Let $\cDn$ denote a product distribution on ${(x_1,y_1),\ldots,(x_n,y_n)}$. Also denote by $Q_{1:n}^{1:t}$ the collection of random variables $Q_{1}^1,\ldots,Q_{n}^1,\ldots, Q_{n}^t$. Let $\langle \cdot,\cdot \rangle$ denote the inner product. We have the following result.
\begin{thm}
  \label{thm:unbiased}
  Let $\hatLth\defeq \frac{1}{nt}\sum_{i=1}^n \sum_{\tau=1}^t \frac{\Qitau}{\pitau} \phi(y_i\langle h,x_i\rangle),$ where $p_i^{\tau}>0$ for all $\tau=1,\ldots,t$. Then 
\begin{equation}  
  \bbE_{Q_{1}^{1},\ldots,Q_{n}^{t},\cDn} \hatLth=L(h).
\end{equation}
\end{thm}
\begin{proof}
  \begin{multline}
    \bbE_{Q_{1:n}^{1:t},\cDn}\hatLth =\bbE_{Q_{1:n}^{1:t},\cDn} \frac{1}{nt}\sum_{i=1}^n\sum_{\tau=1}^t
    \frac{\Qitau}{\pitau} \phi(y_i\langle h,x_i\rangle)\nonumber
    =\bbE_{Q_{1:n}^{1:t},\cDn} \frac{1}{nt}\sum_{i=1}^n\sum_{\tau=1}^t \bbE_{Q_{i}^\tau|Q_{1:n}^{1:\tau-1},\cDn}  \frac{\Qitau}{\pitau} \phi(y_i\langle h,x_i\rangle)
    =\\\bbE_{\cDn} \frac{1}{nt}\sum_{i=1}^n\sum_{\tau=1}^t  \phi(y_i\langle h,x_i\rangle)=L(w).\qedhere
  \end{multline}
\end{proof}
The theorem guarantees that as long as the probability of querying any point in the pool in any round is non-zero $\hat{L}_{t}(h)$, will be an unbiased estimator of $L(h)$. How does one come up with a probability distribution on $\cP$ in round $t$? To solve this problem we resort to probabilistic uncertainty sampling, where the point whose label is most uncertain as per the current hypothesis, $h_{A,t-1}$, gets a higher probability mass. The current hypothesis is simply the minimizer of the importance weighted risk in $\cH$, i.e. $h_{A,t-1}=\arg\min_{h\in\cH} \hat{L}_{t-1}(h)$. For any point $x_i\in\cP$, to calculate the uncertainty of the label $y_i$ of $x_i$, we first estimate $\eta(x_i)\defeq\bbP[y_i=1|x_i]$ using $h_{A,t-1}$, and then use the entropy of the label distribution of $x_i$ to calculate the probability of querying $x_i$. The estimate of $\eta(\cdot)$ in round $t$ depends both on the current active learner $h_{A,t-1}$, and the loss function. In general it is not possible to estimate $\eta(\cdot)$ with arbitrary convex loss functions. However it has been shown by Zhang~\citeyearpar{zhang2004statistical} that the squared, logistic and exponential losses tend to estimate the underlying conditional distribution $\eta(\cdot)$.
Steps 4, 11 of algorithm~\ref{alg:poolal} depend on the loss function $\phi(\cdot)$ being used. If we use the logistic loss i.e $\phi(yz)=\ln(1+\exp(-yz))$ then $\hat{\eta_t}(x)=\frac{1}{1+\exp(-yh_{A,t-1}^Tx)}$. In case of squared loss $\hat{\eta_t}(x)=\min\{\max\{0,w_{A,t-1}^Tx\},1\}$. Since the loss function is convex, and  the constraint set $\cH$ is convex, the minimization problem in step 11 of the algorithm is a convex optimization problem.
\begin{algorithm}[h]
  \caption{\label{alg:poolal}UPAL (Input: $\cP=\{x_1,\ldots,x_n,\}$, Loss function $\phi(\cdot)$, Budget $B$, Labeling Oracle $\cO$)}
  \begin{algorithmic}
    \STATE 1. Set num\_unique\_queries=0, $h_{A,0}=0$, $t=1$.
     \WHILE{num\_unique\_queries~$\leq B$}
         \STATE 2. Set $\Qit=0$ for all $i=1,\ldots,n$.
         \FOR{$x_1,\ldots,x_n\in \cP$}
         \STATE 3. Set $p_{\text{min}}^{t}=\frac{1}{nt^{1/4}}$.
         \STATE 4. Calculate $\hat{\eta_t}(x_i)=\bbP[y=+1|x_i,h_{A,{t-1}}]$.
         \STATE 5. Assign $\pit=p_{\text{min}}^{t}+(1-np_{\text{min}}^t)\frac{\hat{\eta}_t(x_i)\ln(1/\hat{\eta}_t(x))+(1-\hat{\eta}_t(x_i))\ln(1/(1-\hat{\eta}_t(x_i)))}{\sum_{j=1}^n\hat{\eta}_t(x_j)\ln(1/\hat{\eta}_t(x_j))+(1-\hat{\eta}_t(x_j))\ln(1/(1-\hat{\eta}_t(x_j)))}$.
         \ENDFOR
         \STATE 6. Sample a point (say $x_j$) from $p^t(\cdot)$.
         \IF{$x_j$ was queried previously}
         \STATE 7. Reuse its previously queried label $y_j$.
         \ELSE
         \STATE 8. Query oracle $\cO$ for its label $y_j$.
         \STATE 9. \flushleft{num\_unique\_queries $\leftarrow$ num\_unique\_queries+1}.
         \ENDIF
         \STATE 10. Set $Q_j^t=1$.
         \STATE 11. Solve the optimization problem: $h_{A,t}=\arg\min_{h\in \cH} \sum_{i=1}^n\sum_{\tau=1}^t \frac{Q_{i}^{\tau}}{p_i^{\tau}}\phi(y_ih^Tx_i)$.
         \STATE 12. $t\leftarrow t+1$.
    \ENDWHILE
    \STATE 13. Return $h_{A}\defeq h_{A,t}$
  \end{algorithmic}
\end{algorithm} 

By design UPAL might requery points. An alternate strategy is to not allow requerying of points. However the importance weighted risk may not be an unbiased estimator of the true risk in such a case. Hence in order to retain the unbiasedness property we allow requerying in UPAL.

\subsection{The case of squared loss}\label{sec:ewa}
It is interesting to look at the behaviour of UPAL in the case of squared loss where $\phi(yh^Tx)=(1-yh^Tx)^2$. For the rest of the paper we shall denote by $\ha$ the hypothesis returned by UPAL at the end of $T$ rounds. We now show that the prediction of $\ha$ on any $x$ is simply the exponentially weighted average of predictions of all $h$ in $\cH$. 
\begin{thm}\label{thm:ewa}
  Let 
  \begin{align*}
    z_i\defeq&\sum_{t=1}^T\frac{\Qit}{\pit} &\Sighz&\defeq \sum_{i=1}^n z_ix_ix_i^T\\
    v_z\defeq& \sum_{i=1}^{n} z_iy_ix_i  &c\defeq& \sum_{i=1}^n z_i.
  \end{align*}
  Define $w\in\bbR^d$ as 
  \begin{equation}
    \label{eqn:w}
    w=\frac{\int_{\bbR^d} \exp(-\hat{L}_{T}(h))h~\mathrm{d}h}{\int_{\bbR^d}\exp(-\hat{L}_{T}(h))~\mathrm{d}h}.
  \end{equation} 
  Assuming $\Sighz$ is invertible we have for any $x_0\in \bbR^d$, $w^Tx_0=h_A^Tx_0$.
  \begin{proof}
    By elementary linear algebra one can establish that 
    \begin{align}
      \ha&=\Sighzi v_z\label{eqn:ha}\\
      \hat{L}_{T}(h)&=(h-\Sighzi v_z)\Sighz(h-\Sighzi v-z).
    \end{align}
  Using standard integrals we get
  \begin{equation}
    Z\defeq \int_{\bbR^d}\exp(-\hat{L}_{T}(h))~\mathrm{d}h=
    \exp(-c-v_z^T\Sighzi v_z)\sqrt{\pi^d}\sqrt{\det(\Sighzi)}.
  \end{equation}
  In order to calculate $w^Tx_0$, it is now enough to calculate the integral $$I\defeq \int_{\bbR^d} \exp(-\hat{L}_{T}(h))~h^Tx_0~\mathrm{d}w.$$ To solve this integral we proceed as follows. Define $I_1=\int_{\bbR^d}\exp(-\hat{L}_{T}(h))~ h^Tx_0~\mathrm{d}h$. By simple algebra we get 
  \begin{align}
    I&=\int_{\bbR^d} \exp(-w^T\Sighz w+2w^Tv_z-c)~ w^Tx_0~\mathrm{d}w\\
    &=\exp(-c-v_z^T\Sighzi v_z)I_1.\label{eqn:eqn_I}
  \end{align}
Let $a=h-\Sighzi v_z$. We then get 
\begin{align}
  I_1&=\int_{\bbR^d}h^Tx_0\exp\left(-(h-\Sighzi v_z)\Sighz(h-\Sighzi v_z)\right)~\mathrm{d}h\nonumber\\
  &=\int_{\bbR^d} (a^Tx_0+v_z^T\Sighzi x_0)\exp(-a^T\Sighz a)~\mathrm{d}a\nonumber\\
  &=\underbrace{\int_{\bbR^d} (a^Tx_0)\exp(-a^T\Sighz a)~\mathrm{d}a}_{I_2}+\nonumber
 \underbrace{\int_{\bbR^d} v_z^T\Sighzi x_0\exp(-a^T\Sighz a)~\mathrm{d}a}_{I_3}.\label{eqn:eqn_2}
\end{align} 
Clearly $I_2$ being the integrand of an odd function over the entire space calculates to 0. To calculate $I_3$ we shall substitute $\Sighz=SS^T$, where $S\succ 0$. Such a decomposition is possible since $\Sighz\succ 0$. Now define $z=S^Ta$. We get
\begin{align}
  I_3&=v_z^T\Sighzi x_0\int \exp(-z^Tz)~\det(S^{-1})~\mathrm{d}z\\
  &=v_z^T\Sighzi x_0 \det(S^{-1})\sqrt{\pi^{d}}.\label{eqn:eqn_for_I3}
\end{align}
Using equations~(\ref{eqn:eqn_I}, \ref{eqn:eqn_2}, \ref{eqn:eqn_for_I3}) we get
    \begin{gather}
      I=(\sqrt{\pi})^dv_z^T\Sighzi x_0 ~\det(S^{-1})\exp(-c-v_z^T\Sighzi v_z).
    \end{gather} 
    Hence we get
    \begin{equation*}
      w^Tx_0=v_z^T\Sighzi x_0\frac{\det(S^{-1})}{\sqrt{\det(M^{-1})}}=v_z^T\Sighzi x_0=\ha^Tx_0,
    \end{equation*}
    where the penultimate equality follows from the fact that $\det(\Sighzi)=1/\det(\Sighz)=1/(\det(SS^T))=1/(\det(S))^2$, and the last equality follows from equation~\ref{eqn:ha}.
  \end{proof}
\end{thm}
Theorem~\ref{thm:ewa} is instructive. It tells us that assuming that the matrix $\Sighz$ is invertible, $\ha$ is the same as an exponentially weighted average of all the hypothesis in $\cH$. Hence one can view UPAL as learning with expert advice, in the stochastic setting, where each individual hypothesis $h\in\cH$ is an expert, and the exponential of $\hat{L}_{T}$ is used to weigh the hypothesis in $\cH$. Such forecasters have been commonly used in learning with expert advice. This also allows us to interpret UPAL as pruning the hypothesis space in a soft way via exponential weighting, where the hypothesis that has suffered more cumulative loss gets lesser weight.
\section{Bounding the excess risk}\label{sec:consistency}   It is natural to ask if UPAL is consistent? That is will UPAL do as well as the optimal hypothesis in $\cH$ as  $n\rightarrow \infty,T\rightarrow \infty$? We answer this question in affirmative. We shall analyze the excess risk of the hypothesis returned by our active learner, denoted as $h_{A}$, after $T$ rounds when the loss function is the squared loss. The prime motivation for using squared loss over other loss functions is that squared losses yield closed form estimators,  which can then be elegantly analyzed using results from random matrix theory~\citep{hsu2011analysis,hsu2011dimension,tropp2010user}. It should be possible to extend these results to other loss functions such as the logistic loss, or exponential loss using results from empirical process theory~\citep{vandegeer2000empirical}.
\subsection{Main result}
\begin{thm}
  \label{thm:main}
  Let $(x_1,y_1),\ldots (x_n,y_n)$ be sampled i.i.d from a distribution. Suppose assumptions A0-A3 hold. Let $\delta\in(0,1)$, and suppose $n\geq n_{0,\delta},T\geq \max\{T_{0,\delta},T_{1,\delta}\}$. With probability atleast $1-10\delta$ the excess risk of the active learner returned by UPAL after $T$ rounds is 
\begin{equation*}
  L(h_A)-L(\beta)= O\left(\frac{1}{n}+\frac{n}{\sqrt{T}}(d+2\sqrt{d\ln(1/\delta)}+2\ln(1/\delta))\right).
\end{equation*}
\end{thm}
\subsection{Assumptions, and Notation.}
\begin{enumerate}
\item[\textbf{A0}] (Invertibility of $\Sigma$) The data covariance matrix $\Sigma$ is invertible.
\item [\textbf{A1}] (Statistical leverage condition) There exists a finite $\gamma_0\geq 1$ such that almost surely 
  \begin{equation*}
    ||\Sig^{-1/2}{x}||\leq \gamma_0\sqrt{d}.
  \end{equation*}
\item[\textbf{A2}] There exists a finite $\gamma_1\geq 1$ such that $\bbE[\exp(\alpha^Tx)]\leq \exp\left(\frac{||\alpha||^2\gamma_1^2}{2}\right)$.
\item[\textbf{A3}] (Linear hypothesis) We shall assume that $y=\beta^Tx+\xi(x)$, where $\xi(x)\in [-2,+2]$ is additive noise with $\bbE[\xi(x)|x]=0$.
\end{enumerate}
Assumption A0 is necessary for the problem to be well defined. A1 has been used in recent literature to analyze linear regression under random design and is a Bernstein like condition~\citep{rokhlin2008fast}. A2 can be seen as a softer form of boundedness condtion on the support of the distribution. In particular if the data is bounded in a d-dimensional unit cube then it suffices to take $\gamma_1=1/2$. It may be possible to satisfy A3 by mapping data to kernel spaces. Though popularly used kernels such as Gaussian kernel map the data to infinite dimensional spaces, a finite dimensional approximation of such kernel mappings can be found by the use of random features~\citep{rahimi2007random}. 
\newline \newline
\textbf{Notation.}
\begin{enumerate}
\item $h_A$ is the active learner outputted by our active learning algorithm at the end of $T$ rounds.
\item 
  \begin{align*}
    \forall i=1,\ldots,n: z_i\defeq \sum_{t=1}^T\frac{\Qit}{\pit} &\hspace{30pt} \Sighz\defeq\sum_{i=1}^n z_ix_ix_i^T\\
    \psi_z\defeq \sum_{i=1}^nz_i\xi(x_i)x_i& \hspace{30pt}\Sigh\defeq\frac{1}{n}\sum_{i=1}^n x_ix_i^T\\
    \Sig\defeq\bbE[xx^T]& \hspace{30pt}\Sighz\defeq \sum_{i=1}^n z_ix_ix_i^T\\
    n_{0,\delta}\defeq 7200d^2\gamma_0^4(d\ln(5)+\ln(10/\delta))&\hspace{30pt}
    T_{1,\delta}\defeq 12+512\sqrt{2}d^{8/3}\gamma_0^{16/3}\ln^{4/3}(d/\delta)
  \end{align*}
$$T_{0,\delta}\defeq \gamma_1^{16/3}d^{8/3}\ln^{4/3}(d/\delta)\ln^{8/3}(n/\delta)\lambdamin^{8/3}(\Sigma)+4\ln(d/\delta)\frac{\lambdamax(\Sigma)}{\lambdamin(\Sigma)},$$ where $\delta\in(0,1)$.
\end{enumerate}
\subsection{Overview of the proof}
The excess risk of a hypothesis $h\in\cH$ is defined as $L(h)-L(\beta)=\bbE_{x, y\sim\cD} [(y-h^Tx)^2-(y-\beta^Tx)^2]$. Our aim is to provide high probability bounds for the excess risk, where the probability measure is w.r.t the sampled points $(x_1,y_1),\ldots,(x_n,y_n), Q_1^1,\ldots,Q_{n}^T$. The proof proceeds as follows.
\begin{enumerate}
\item In lemma~\ref{lem:decompose}, assuming that the matrices $\Sighz,\Sigh$ are invertible we upper bound the excess risk as the product $||\Sig^{1/2}\Sighzi\Sig^{1/2}||^2||\Sig^{-1/2}\Sigh^{1/2}||^2 ||\Sigh^{-1/2}\psi_z||^2$. The prime motivation in doing so is that bounding such ``squared norm'' terms can be reduced to bounding the maximum eigenvalue of random matrices, which is a well studied problem in random matrix theory.
\item In lemma~\ref{lem:1} we provide an upper bound for $||\Sig^{-1/2}\Sigh^{1/2}||^2$. To do this we use the simple fact that the matrix 2-norm of a positive semidefinite matrix is nothing but the maximum eigenvalue of the matrix. With this obsercation, and by exploiting the structure of the matrix $\Sigh$, the problem reduces to giving probabilistic upper bounds for maximum eigenvalue of a sum of random rank-1 matrices. Theorem~\ref{thm:litvak} provides us with a tool to prove such bounds.
\item In lemma~\ref{lem:2} we bound $||\Sig^{1/2}\Sighzi\Sig^{1/2}||^2$. The proof is in the same spirit as in lemma~\ref{lem:1}, however the resulting probability problem is that of bounding the maximum eigenvalue of a sum of random matrices, which are not necessarily rank-1. Theorem~\ref{thm:mat_bern} provides us with Bernstein type bounds to analyze the eigenvalues of sums of random matrices.
\item In lemma~\ref{lem:3} we bound the quantity $||\Sigh^{-1/2}\psi_z||^2$. Notice that here we are bounding the squared norm of a random vector. Theorem~\ref{thm:quadratic} provides us with a tool to analyze such quadratic forms under the assumption that the random vector has sub-Gaussian exponential moments behaviour.
\item Finally all the above steps were conditioned on the invertibility of the random matrices $\Sigh,\Sighz$. We provide conditions on $n,T$ (this explains why we defined the quantities $n_{0,\delta},T_{0,\delta},T_{1,\delta}$) which guarantee the invertibility of $\Sigh,\Sighz$. Such problems boil down to calculating lower bounds on the minimum eigenvalue of the  random matrices in question, and to establish such lower bounds we once again use theorems~\ref{thm:litvak},~\ref{thm:mat_bern}.
\end{enumerate}
\subsection{Full Proof}
We shall now provide a way to bound the excess risk of our active learner hypothesis. Suppose $\ha$ was the hypothesis represented by the active learner at the end of the T rounds.  
By the definition of our active learner and the definition of $\beta$ we get
\begin{align}
\ha&=\arg\min_{h\in\cH} ~\sum_{i=1}^n \sum_{t=1}^T\frac{\Qit}{\pit} (y_i-h^Tx_i)^2=\sum_{i=1}^n z_i (y_i-h^Tx_i)^2=\Sighzi v_z\\
\beta&=\arg\min_{h\in\cH}\bbE(y-\beta^Tx)^2=\Sigi\bbE[yx].
\end{align}
\begin{lemma}
\label{lem:decompose}
Asumme $\Sighz,\Sigh$ are both invertible, and assumption A0 applies. Then the excess risk of the classifier after $T$ rounds of our active learning algorithm is given by 
\begin{equation}
\label{eqn:decompose}
L(h_{A})-L(\beta)\leq ||\Sig^{1/2}\Sighzi\Sig^{1/2}||^2||\Sig^{-1/2}\Sigh^{1/2}||^2 ||\Sigh^{-1/2}\psi_z||^2.
\end{equation}
\end{lemma}
\begin{proof}
  \begin{align}
    L(\ha)-L(\beta)&=\bbE[(y-\ha^Tx)^2-(y-\beta^Tx)^2]\nonumber\\
    &=\bbE_{x,y}[\ha^Txx^T\ha-2y\ha^Tx-\beta^Txx^T\beta+2y\beta^Tx]\nonumber\\
    &=\ha^T\Sigma \ha-2\ha^T\bbE[xy]-\beta^T\Sigma\beta+2\beta^T\Sigma\beta 
    \text{~[Since $\Sig\beta=\bbE[yx]$]}\nonumber\\
    &=\ha^T\Sigma \ha-\beta^T\Sigma\beta-2\ha^T\Sigma\beta+2\beta^T\Sigma\beta \nonumber\\
    &=\ha^T\Sigma \ha+\beta^T\Sigma\beta-2\ha^T\Sigma\beta \nonumber\\
    &=||\Sigma^{1/2}(\ha-\beta)||^2\label{eqn:exrisk1}.
  \end{align}
  We shall next bound the quantity $||\ha-\beta||$ which will be used to bound the excess risk in Equation~(
  \ref{eqn:exrisk1}). To do this we shall use assumption A3 along with the definitions  of $\ha,\beta$. We have the following chain of inequalities.
\begin{align}
\ha&=\Sighzi v_z\nonumber\\
&=\Sighzi\sum_{i=1}^n z_i y_ix_i\nonumber\\
&=\Sighzi\sum_{i=1}^n z_i(\beta^Tx_i+\xi(x_i))x_i\nonumber\\
&=\Sighzi\sum_{i=1}^n z_i x_ix_i^T\beta+z_i\xi(x_i)x_i\nonumber\\
&=\beta+\Sighzi\sum_{i=1}^n z_i\xi(x_i)x_i=\beta+\Sighzi\psi_z.\label{eqn:habd}
\end{align} 
Using Equations ~\ref{eqn:exrisk1},\ref{eqn:habd} we get the following series of inequalities for the excess risk bound
\begin{align}
L(\ha)-L(\beta)&=||\Sig^{1/2}\Sighzi\psiz||^2\nonumber\\
&=||\Sig^{1/2}\Sighzi\Sigh^{1/2}\Sigh^{-1/2}\psiz||^2\nonumber\\
&=||\Sig^{1/2}\Sighzi\Sig^{1/2}\Sig^{-1/2}\Sigh^{1/2}\Sigh^{-1/2}\psiz||^2\\
&\leq ||\Sig^{1/2}\Sighzi\Sig^{1/2}||^2||\Sig^{-1/2}\Sigh^{1/2}||^2 ||\Sigh^{-1/2}\psi_z||^2.\qedhere
\end{align}
\end{proof}
The decomposition in lemma~\ref{lem:decompose} assumes that both $\Sighz,\Sigh$ are invertible. Before we can establish conditions for the matrices $\Sighz,\Sigh$ to be invertible we need the following elementary result.
\begin{proposition}
  \label{prop:expmoments_statlev}
  For any arbitrary $\alpha\in \bbR^d$, under assumption A1 we have
\begin{equation}
  \bbE[\exp(\alpha^T\Sig^{-1/2}x)]\leq 5\exp\left(\frac{3d\gamma_0^2||\alpha||^2}{2}\right).
\end{equation}
\end{proposition}
\begin{proof}
  From Cauchy-Schwarz inequality and A1 we get 
  \begin{equation}
   -||\alpha||\gamma_0\sqrt{d} \leq -||\alpha||~||\Sig^{-1/2}x|| \leq \alpha^T\Sig^{-1/2}x\leq ||\alpha||~||\Sig^{-1/2}x||\leq ||\alpha||\gamma_0\sqrt{d} .
\end{equation}
Also $\bbE[\alpha^T\Sig^{-1/2}x]\leq ||\alpha||\gamma_0\sqrt{d}$. Using Hoeffding's lemma we get
\begin{align}
  \bbE[\exp(\alpha^T\Sig^{-1/2}x)]&\leq \exp\left(||\alpha||\gamma_0\sqrt{d}+\frac{||\alpha||^2d\gamma_0^2}{2}\right)\\
  &\leq 5\exp(3||\alpha||^2d\gamma_0^2/2).\qedhere
\end{align}
\end{proof}
The following lemma will be useful in bounding the terms $||\Sig^{1/2}\Sighzi\Sig^{1/2}||$, $||\Sig^{-1/2}\Sigh^{1/2}||^2$.
\begin{lemma}
  \label{lem:J}
  Let $J\defeq \sum_{i=1}^n \Sigma^{-1/2}x_ix_i^T\Sigma^{-1/2}$. Let $n\geq n_{0,\delta}$. Then the following inequalities hold separately with probability atleast $1-\delta$ each
  \begin{align}
    \lambdamax(J)\leq n+6dn\gamma_0^2\left[\sqrt{\frac{32(d\ln(5)+\ln(10/\delta))}{n}}+\frac{2(d\ln(5)+\ln(10/\delta))}{n}\right]\leq 3n/2\\
    \lambdamin(J)\geq n-6dn\gamma_0^2\left[\sqrt{\frac{32(d\ln(5)+\ln(10/\delta))}{n}}+\frac{2(d\ln(5)+\ln(10/\delta))}{n}\right]\geq n/2.
    \end{align}
\end{lemma}
\begin{proof}
Notice that $\bbE[\Sig^{-1/2}x_ix_i^T\Sig^{-1/2}]=I.$ From Proposition~\ref{prop:expmoments_statlev} we have $\bbE[\exp(\alpha^T\Sig^{-1/2}x)]\leq 5\exp(3||\alpha||^2 d\gamma_0^2/2)$. By using theorem~\ref{thm:litvak} we get with probability atleast $1-\delta$:
\begin{equation}
  \lambdamax\left(\frac{1}{n}\sum_{i=1}^n (\Sig^{-1/2}x_i)(\Sig^{-1/2}x_i)^T\right)\leq 1+6d\gamma_0^2\left[\sqrt{\frac{32(d\ln(5)+\ln(2/\delta))}{n}}+\frac{2(d\ln(5)+\ln(2/\delta))}{n}\right].
\end{equation}
 Put $n\geq n_{0,\delta}$ to get the desired result. The lower bound on $\lambdamin$ is also obtained in the same way.
\end{proof}
\begin{lemma}
  \label{lem:inv_sigh} 
Let $n\geq n_{0,\delta}$. With probability atleast $1-\delta$ separately we have $\Sigh\succ 0$, $\lambdamin(\Sigh)\geq \frac{1}{2}\lambdamin(\Sigma)$, $\lambdamax(\Sigh)\leq \frac{3}{2}\lambdamax(\Sigma)$.
\end{lemma}
\begin{proof}
Using lemma~\ref{lem:J} we get for $n\geq n_{0,\delta}$ with probability atleast $1-\delta$, $\lambdamin(J)\geq 1/2$ and with probability atleast $1-\delta$,  $\lambdamax(\Sigma)\leq 3/2$.
Finally since $\Sigma^{1/2}J\Sigma^{1/2}=\Sigh$, and $J\succ 0,\Sigma\succ 0$, we get $\Sigh\succ 0$. Further we have the following upper bound with probability atleast $1-\delta$:
\begin{align}
  \label{eqn:sig_ub}
  \lambdamax(\Sigh)&=||\Sigma^{1/2}J\Sigma^{1/2}||\\
  &\leq ||\Sigma^{1/2}||^2~||J|| \\
  &\leq ||\Sigma||~||J||\\
  &=\lambdamax(\Sigma)\lambdamax(J)\\
  &\leq \frac{3}{2} \lambdamax(\Sigma),
\end{align}
 where in the last step we used the upper bound on $\lambdamax(J)$ provided by lemma~\ref{lem:J}.
Similarly we have the following lower bound with probability atleast $1-\delta$
\begin{align}
  \label{eqn:sig_lb}
  \lambdamin(\Sigh)&=\frac{1}{\lambdamax(\Sig^{-1/2}J^{-1}\Sig^{-1/2})}\\
  &=\frac{1}{||\Sig^{-1/2}J^{-1}\Sig^{-1/2}||}\\
  &\geq \frac{1}{||\Sig^{-1}||~||J^{-1}||~||\Sig^{-1/2}||}\\
  &=\lambdamin(\Sigma)\lambdamin(J)\\
  &\geq\frac{\lambdamin(\Sigma)}{2},
\end{align}
where in the last step we used the lower bound on $\lambdamin(J)$ provided by lemma~\ref{lem:J}.\qedhere
\end{proof}
The following proposition will be useful in proving lemma~\ref{lem:inv_sighz}.
\begin{proposition}
\label{prop:normxibound}
  Let $\delta\in(0,1)$. Under assumption A2, with probability atleast $1-\delta$, $\sum_{i=1}^n||x_i||^4\leq 25\gamma_1^4d^2\ln^2(n/\delta)$
\end{proposition}
\begin{proof}
From A2 we have $\bbE[\exp(\alpha^Tx)]\leq \exp(\frac{||\alpha||^2\gamma_1^2}{2})$. Now applying theorem~\ref{thm:quadratic} with $A=I_{d}$ we get 
\begin{equation}
  \bbP[||x_i||^2\leq d\gamma_1^2+2\gamma_1^2\sqrt{d\ln(1/\delta)}+2\gamma_1^2\ln(1/\delta)]\geq 1-\delta.
\end{equation}
The result now follows by the union bound.
\end{proof}
\begin{lemma}
  \label{lem:inv_sighz}
  Let $\delta\in(0,1)$. For $T\geq T_{0,\delta}$, with probability atleast $1-4\delta$ we have  $\lambdamin(\Sighz)\geq \frac{nT\lambdamin(\Sigma)}{4}>0$. Hence $\Sighz$ is invertible.
\end{lemma}
\begin{proof}
The proof uses theorem~\ref{thm:mat_bern}.
Let $M_t'\defeq \sum_{i=1}^n \frac{\Qit}{\pit} x_ix_i^T$, so that $\Sighz=\sum_{t=1}^T M_t'$. Now $\bbE_tM_t'=n\Sigh$. Define $R_t'\defeq n\Sigh-M_t'$, so that $\bbE_t R_t'=0$. We shall apply theorem~\ref{thm:mat_bern} to the random matrix $\sum R_t'$. In order to do so we need upper bounds on $\lambdamax (R_t')$ and $\lambdamax (\frac{1}{T}\sum_{t=1}^T \bbE_t R_t'^2)$. Let $n\geq n_{0,\delta}$. Using lemma~\ref{lem:inv_sigh} we get with probability atleast $1-\delta$ 
\begin{align}
  \lambdamax(R_t')=\lambdamax(n\Sigh-M_t')\leq \lambdamax(n\Sigh)\leq \frac{3n\lambdamax(\Sig)}{2}\defeq b_2.
\end{align}
\begin{align}
  \lambdamax\left[\frac{1}{T}\sum_{t=1}^T \bbE_t R_t'^2\right]&=\frac{1}{T}\lambdamax\left[\sum_{t=1}^T \bbE_t(n\Sigh-M_t')^2\right]\label{eqn:jump-2}\\
  &=\frac{1}{T}\lambdamax(-n^2T\Sigh^2+\sum_{t=1}^T\bbE_t\sum_{i=1}^n \frac{\Qit}{(\pit)^2}(x_ix_i^T)^2)\label{eqn:jump-1}\\
  &=\frac{1}{T}\lambdamax(-n^2T\Sigh^2+\sum_{t=1}^T\sum_{i=1}^n \frac{1}{\pit}(x_ix_i^T)^2)\label{eqn:jump0}\\
&\leq \frac{1}{T}\lambdamax(\sum_{i=1}^n\sum_{t=1}^T\frac{1}{\pit}(x_ix_i^T)^2)-n^2\lambdamin^2(\Sigh)\label{eqn:jump1}\\
  &\leq nT^{1/4}\lambdamax(\sum_{i=1}^n (x_ix_i^T)^2)\label{eqn:jump2}\\
  &\leq nT^{1/4}\sum_{i=1}^n\lambdamax^2(x_ix_i^T)\label{eqn:jump3}\\
  &=nT^{1/4}\sum_{i=1}^n ||x_i||^4\label{eqn:jump4}\\
  &\leq 25\gamma_1^4d^2n^2T^{1/4}\ln^2(n/\delta)\defeq \sigma_2^2\label{eqn:jump5}.
\end{align}
Equation~\ref{eqn:jump-1} follows from Equation~\ref{eqn:jump-2} by the definition of $M_t'$ and the fact that at any given $t$ only one point is queried i.e. $\Qit Q_{j}^t=0$ for a given $t$. Equation~\ref{eqn:jump0} follows from equation~\ref{eqn:jump-1} since $E_{t}\Qit=\pit$. Equation~\ref{eqn:jump1} follows from Equation~\ref{eqn:jump0} by Weyl's inequality. Equation~\ref{eqn:jump2} follows from Equation~\ref{eqn:jump1} by substituting $p_{\text{min}}^t$ in place of $\pit$. Equation~\ref{eqn:jump3} follows from Equation~\ref{eqn:jump2} by the use of Weyl's inequality. Equation~\ref{eqn:jump4} follows from Equation~\ref{eqn:jump3} by using the fact that if $p$ is a vector then $\lambdamax(pp^T)=||p||^2$. Equation~\ref{eqn:jump5} follows from Equation~\ref{eqn:jump4} by the use of proposition~\ref{prop:normxibound}. Notice that this step is a stochastic inequality and holds with probability atleast $1-\delta$.

Finally applying theorem~\ref{thm:mat_bern} we have
\begin{align}
  \bbP\left[\lambdamax(\frac{1}{T}\sum_{t=1}^T R_t')\leq \sqrt{\frac{2\sigma_2^2\ln(d/\delta)}{T}}+\frac{b_2\ln(d/\delta)}{T}\right]\geq 1-\delta\\
\implies \bbP\left[\lambdamax(n\Sigh-\frac{1}{T}\sum_{t=1}^TM_t')\leq \sqrt{\frac{2\sigma_2^2\ln(d/\delta)}{T}}+\frac{b_2\ln(d/\delta)}{T}\right]\geq 1-\delta\\
\implies \bbP\left[\lambdamin(n\Sigh)-\frac{1}{T}\lambdamin\left(\sum_{t=1}^TM_t'\right)\leq \sqrt{\frac{2\sigma_2^2\ln(d/\delta)}{T}}+\frac{b_2\ln(d/\delta)}{T}\right]\geq 1-\delta
\end{align}
Substituting for $\sigma_2,b_2$, rearranging the inequalities, and using lemma~\ref{lem:inv_sigh} to lower bound $\lambdamin(\Sigh)$ we get
\begin{align*}
  \bbP\left[\lambdamin(\sum_{t=1}^TM_t')\geq T\lambdamin(n\Sigh)-\sqrt{2T\sigma_2^2\ln(d/\delta)}-b_2\ln(d/\delta)\right]\geq 1-\delta\\
  \implies \bbP\left[\lambdamin(\sum_{t=1}^TM_t')\geq \frac{nT\lambdamin(\Sigma)}{2}-\sqrt{2T\sigma_2^2\ln(d/\delta)}-b_2\ln(d/\delta)\right]\geq 1-2\delta\\
  \implies \bbP\left[\lambdamin(\sum_{t=1}^TM_t')\geq\frac{nT\lambdamin(\Sigma)}{2}-5\sqrt{2}\gamma_1^2dnT^{5/8}\sqrt{\ln(d/\delta)}\ln(n/\delta)-\frac{n\ln(d/\delta)\lambdamax(\Sigma)}{2}\right]\geq 1-4\delta
\end{align*}
For $T\geq T_{0,\delta}$  with probability atleast $1-4\delta$,
$\lambdamin\sum_{t=1}^TM_t'=\lambdamin(\Sighz)\geq\frac{nT\lambdamin(\Sigma)}{4}$.
\end{proof}
\begin{lemma}
  \label{lem:1}
  For $n\geq n_{0,\delta}$ with probability atleast $1-\delta$ over the random sample $x_1,\ldots,x_n$
  \begin{equation}
    ||\Sig^{-1/2}\Sigh^{1/2}||^2\leq 3/2.
  \end{equation}
\end{lemma}
\begin{proof}
  \begin{align}
    ||\Sig^{-1/2}\Sigh^{1/2}||^2&=||\Sigh^{1/2}\Sig^{-1/2}||^2\\
    &=\lambdamax(\Sig^{-1/2}\Sigh\Sig^{-1/2})\\
    &=\lambdamax\left(\frac{1}{n}\sum_{i=1}^n (\Sig^{-1/2}x_i)(\Sig^{-1/2}x_i)^T\right)\\
    &=\lambdamax\left(\frac{J}{n}\right)\\
    &\leq 3/2
  \end{align}
  where in the first equality we used the fact that $||A||=||A^T||$ for a square matrix $A$, and $||A||^2=\lambdamax(A^TA)$, and in the last step we used lemma~\ref{lem:J}.
\end{proof}
\begin{lemma}
  \label{lem:2}
Suppose $\Sighz$ is invertible. Given $\delta\in (0,1)$, for $n\geq n_{0,\delta}$, and $T\geq \max\{T_{0,\delta}, T_{1,\delta}\}$ with probability atleast $1-3\delta$ over the samples 
\begin{equation*}
  ||\Sig^{1/2}\Sighzi\Sig^{1/2}||^2\leq\frac{400}{n^2T^2}.
\end{equation*}
\end{lemma}
\begin{proof}
The proof of this lemma is very similar to the proof of lemma~\ref{lem:inv_sighz}. From lemma~\ref{lem:inv_sighz} for $n\geq n_{0,\delta}, T\geq T_{0,\delta}$ with probability atleast $1-\delta$, $\Sighz\succ 0$. Using the assumption that $\Sig\succ 0$, we get  $\Sig^{1/2}\Sighzi\Sig^{1/2}\succ 0$. Hence $||\Sig^{1/2}\Sighzi\Sig^{1/2}||=\lambdamax(\Sig^{1/2}\Sighzi\Sig^{1/2})=\frac{1}{\lambdamin(\Sig^{-1/2}\Sighz\Sig^{-1/2})}$. Hence it is enough to provide a lower bound on the smallest eigenvalue of the symmetric positive definite matrix $\Sig^{-1/2}\Sighz\Sig^{-1/2}$. 
\begin{align*}
  \lambdamin(\Sig^{-1/2}\Sighz\Sig^{-1/2})&=\lambdamin\left(\sum_{i=1}^n z_i\Sig^{-1/2} x_ix_i^T\Sig^{-1/2}\right)\\
  &=\lambdamin(\sum_{t=1}^T\underbrace{\sum_{i=1}^n \frac{\Qit}{\pit}\Sig^{-1/2} x_ix_i^T\Sig^{-1/2}}_{\defeq M_t})\\
  &=\lambdamin\left(\sum_{t=1}^T M_t\right).
\end{align*}
Define $R_t\defeq J-M_t$. Clearly $\bbE_{t}[M_t]=J$, and hence $\bbE[R_t]=0$. From Weyl's inequality we have $\lambdamin(J)+\lambdamax\left(\frac{-1}{T}\sum_{t=1}^T M_t\right)\leq \lambdamax(\frac{1}{T}\sum_{t=1}^T R_t)$. Now applying theorem~\ref{thm:mat_bern} on $\sum R_t$ we get with probability atleast $1-\delta$
\begin{equation}\label{eqn:matbern}
\lambdamin(J)+\lambdamax\left(\frac{-1}{T}\sum_{t=1}^T M_t\right)\leq \lambdamax\left(\frac{1}{T}\sum_{t=1}^T R_t\right)\leq \sqrt{\frac{2\sigma_
1^2\ln(d/\delta)}{T}}+\frac{b_1\ln(d/\delta)}{3T},
\end{equation}
where 
\begin{align}
  \lambdamax\left(\frac{1}{T}\sum_{t=1}^T J-M_t\right)\leq b_1\\
  \lambdamax\left(\frac{1}{T}\sum_{t=1}^T \bbE_{t}(J-M_t)^2\right)\leq \sigma_1^2
\end{align}
Rearranging Equation~(\ref{eqn:matbern}) and using the fact that $\lambdamax(-A)=-\lambdamin(A)$ we get with probability atleast $1-\delta$,
\begin{equation}\label{eqn:matbernrear}
  \lambdamin\left(\sum_{t=1}^T M_t\right)\geq T\lambdamin(J)-\sqrt{2T\sigma_
1^2\ln(d/\delta)}-\frac{b_1\ln(d/\delta)}{3}.
\end{equation}
Using Weyl's inequality~\citep{horn90matrix} we have $\lambdamax(\frac{1}{T}\sum_{t=1}^T J-M_t)\leq \lambdamax(J)\leq \frac{3n}{2}$ with probability atleast $1-\delta$, where in the last step we used lemma~(\ref{lem:J}). Let $b_1\defeq\frac{3n}{2}$. To calculate $\sigma_1^2$ we proceed as follows.

\begin{align}
\lambdamax\left(\frac{1}{T}\sum_{t=1}^T \bbE_{t}(J-M_t)^2\right)&=\frac{1}{T}\lambdamax\left(\sum_{t=1}^T \bbE_{t}(M_t^2)-J^2\right)\label{eqn:here0}\\
&\leq \frac{1}{T}\lambdamax\left(\sum_{t=1}^T \bbE_{t}M_t^2\right)\label{eqn:here1}\\
&=\frac{1}{T}\lambdamax\left(\sum_{t=1}^T \bbE_t \left(\sum_{i=1}^n \frac{\Qit}{\pit}\Sigma^{-1/2}x_ix_i^T\Sigma^{-1/2}\right)^2\right)\label{eqn:here2}\\
&=\frac{1}{T}\lambdamax\left(\sum_{t=1}^T \bbE_t \sum_{i=1}^n \frac{\Qit}{(\pit)^2}(\Sigma^{-1/2}x_ix_i^T\Sigma^{-1/2})^2\right)\label{eqn:here3}\\
&=\frac{1}{T}\lambdamax\left(\sum_{t=1}^T \sum_{i=1}^n \frac{1}{\pit}(\Sigma^{-1/2}x_ix_i^T\Sigma^{-1/2})^2\right)\label{eqn:here4}\\
&\leq \frac{1}{T}\sum_{t=1}^T\sum_{i=1}^n \frac{1}{\pit}||\Sig^{-1/2}x_i||^4\label{eqn:here5}\\
&\leq \frac{d^2\gamma_0^4}{T}\sum_{i=1}^n\sum_{t=1}^T \frac{1}{\pit}\label{eqn:here6}\\
&\leq \frac{nd^2\gamma_0^4}{T}\sum_{t=1}^T \frac{1}{p_{\text{min}}^t}\label{eqn:here7}\\
&\leq n^2d^2\gamma_0^4 T^{1/4}\defeq \sigma_1^2\label{eqn:here8}.
\end{align}
Equation~\ref{eqn:here1} follows from Equation~\ref{eqn:here0} by using Weyl's inequality and the fact that $J^2\succeq 0$. Equation ~\ref{eqn:here3} follows from Equation~\ref{eqn:here2} since only one point is queried in every round and hence for any given $t,i\neq j$ we have $\Qit Q_{j}^t=0$, and hence all the cross terms disappear when we expand the square. Equation~(\ref{eqn:here4}) follows from Equation~(\ref{eqn:here3}) by using the fact that $\bbE_{t}Q_t=p_t$. Equation~(\ref{eqn:here5}) follows from Equation~(\ref{eqn:here4}) by Weyl's inequality and the fact that the maximum eigenvalue of a rank-1 matrix of the form $vv^T$ is $||v||^2$. Equation~(\ref{eqn:here6}) follows from Equation~(\ref{eqn:here5}) by using assumption A1. Equation~\ref{eqn:here8} follows from Equation~(\ref{eqn:here7}) by our choice of $p_{min}^t=\frac{1}{n\sqrt{t}}$. 
Substituting the values of $\sigma_1^2, b_1$ in~\ref{eqn:matbernrear}, using lemma~\ref{lem:J} to lower bound $\lambdamin(J)$, and applying union bound to sum up all the failure probabilities we get for $n\geq n_{0,\delta},T\geq \max\{T_{0,\delta},T_{1,\delta}\}$ with probability atleast $1-3\delta$,
\begin{multline*}
  \lambdamin\left(\sum_{t=1}^T M_t\right)\geq T\lambdamin(J)-\sqrt{2T^{5/4}n^2d^2\gamma_0^4\ln(d/\delta)}-3n/2\\
  \geq \frac{nT}{2}-\sqrt{2}T^{5/8}nd\gamma_0^2\sqrt{\ln(d/\delta)}-3n/2\geq nT/4.\qedhere
\end{multline*}
\end{proof}
The only missing piece in the proof is an upper bound for the quantity $||\Sigh^{-1/2}\psi_z||^2$. The next lemma provides us with an upper bound for this quantity.
\begin{lemma}
  \label{lem:3}
  Suppose $\Sigh$ is invertible. Let $\delta\in (0,1)$. With probability atleast $1-\delta$  we have 
\begin{equation*}
  ||\Sigh^{-1/2}\psi_z||^2\leq (2nT^2+56n^3T\sqrt{T})(d+2\sqrt{d\ln(1/\delta)}+2\ln(1/\delta)).
\end{equation*}
\end{lemma}
\begin{proof}
 Define the matrix $A\in \bbR^{d\times n}$ as follows. Let the $i^{\text{th}}$ column of $A$ be the vector $\frac{\Sigh^{-1/2}x_i}{\sqrt{n}}$, so that $AA^T=\frac{1}{n}\Sigh^{-1/2}x_ix_i^T\Sigh^{-1/2}=I_d$. Now $||\Sigh^{-1/2}\psi_z||^2=||\sqrt{n}Ap||^2$, where $p=(p_1,\ldots,p_n)\in \bbR^n$ and $p_i=\xi(x_i)z_i$ for $i=1,\ldots,n$. Using the result for quadratic forms of subgaussian random vectors (threorem~\ref{thm:quadratic}) we get
\begin{multline}
  \label{eqn:norm_Ap2}
  ||Ap||^2\leq \sigma^2(\tr(I_d)+2\sqrt{\tr(I_d)\ln(1/\delta)}+2||I_d||\ln(1/\delta))=\sigma^2(d+2\sqrt{d\ln(1/\delta)}+2\ln(1/\delta)), 
\end{multline}
where for any arbitrary vector $\alpha$, $\bbE[\exp(\alpha^Tp)]\leq \exp(||\alpha||^2\sigma^2)$.
\end{proof}
Hence all that is left to be done is prove that $\alpha^Tp$ has sub-Gaussian exponential moments. Let
\begin{equation}
D_t\defeq \sum_{i=1}^n \frac{\alpha_i\xi(x_i)\Qit}{\pit}-\alpha^T\xi~~~\forall t=1,\ldots,T.
\end{equation}
With this definition we have the following series of equalities
\begin{align}
  \label{eqn:decompose_condexpec}
  \bbE[\exp(\alpha^Tp)]=\bbE[\exp(\sum D_t+T\alpha^T\xi)]=\bbE\left[\exp(T\alpha^T\xi)\bbE[\exp(\sum D_t)|\cDn]\right].
\end{align}
Conditioned on the data, the sequence $D_1,\ldots,D_T$, forms a martingale difference sequence. Let $\xi=[\xi(x_1),\ldots,\xi(x_n)]$. Notice that 
\begin{equation}
  \label{eqn:bound_Dt}
  -\alpha^T\xi-\frac{2||\alpha||}{p_{\text{min}}^t}\leq D_t\leq -\alpha^T\xi+\frac{2||\alpha||}{p_{\text{min}}^t}.
\end{equation}
We shall now bound the probability of large deviations of $D_t$ given history up until time $t$. This allows us to put a bound on the large deviations of the martingale sum $\sum_{t=1}^T D_t$. Let $a\geq 0$. Using Markov's inequality we get
\begin{align}
  \bbP[D_t\geq a|Q_{1:n}^{1:t-1},\cDn]&\leq \min_{\gamma>0}~\exp(-\gamma a)\bbE[\gamma D_t|Q_{1:n}^{1:t-1},\cDn]\\
  &\leq\min_{\gamma>0}\exp\left(\frac{2\gamma^2||\alpha||^2}{(p_{\text{min}}^t)^2}-\gamma a\right)\\
  & \leq \exp\left(\frac{-a^2}{8||\alpha||^2n^2\sqrt{t}}\right).
\end{align}
In the second step we used Hoeffding's lemma along with the boundedness property of $D_t$ shown in equation~\ref{eqn:bound_Dt}. The same upper bound can be shown for the quantity $\bbP[D_t\leq a|Q_{1:n}^{1:t-1},\cDn]$.
Applying lemma~\ref{lem:modazuma} we get with probability atleast $1-\delta$, conditioned on the data, we have
\begin{equation}
  \frac{1}{T}\sum_{t=1}^T D_t\leq \sqrt{\frac{448||\alpha||^2n^2\ln(1/\delta)}{\sqrt{T}}}\\\implies
  \sum_{t=1}^T D_t\leq \sqrt{112||\alpha||^2n^2T^{3/2}\ln(1/\delta)}.
\end{equation}
Hence $\sum_{t=1}^T D_t$, conditioned on data, has sub-Gaussian tails as shown above. This leads to the following conditional exponential moments bound
\begin{equation}
  \label{eqn:sumdt}
  \bbE[\exp(\sum_{t=1}^T D_t)|\cD_n]=\exp\left(56||\alpha||^2n^2T\sqrt{T}\ln(1/\delta)\right).
\end{equation}
Finally putting together equations~\ref{eqn:decompose_condexpec},~\ref{eqn:sumdt} we get 
\begin{equation}
 \bbE[\exp(\alpha^Tp)]\leq \bbE\exp(T\alpha^T\xi)\exp(56||\alpha||^2n^2T\sqrt{T})\leq \exp((2T^2+56n^2T\sqrt{T})||\alpha||^2),
\end{equation}
In the last step we exploited the fact that $-2 \leq \xi(x_i)\leq 2$, and hence by Hoeffding lemma $\bbE[\exp(\alpha^T\xi)]\leq \exp(2||\alpha||^2)$.
This leads us to the choice of $\sigma^2=2T^2+56n^2T\sqrt{T}$. Substituting this value of $\sigma^2$ in equation~\ref{eqn:norm_Ap2} we get 
\begin{equation}
  ||Ap||^2\leq (2T^2+56n^2T\sqrt{T})(d+2\sqrt{d\ln(1/\delta)}+2\ln(1/\delta)), 
\end{equation}
and hence with probability atleast $1-\delta$,
\begin{equation}
  ||\Sigh^{-1/2}\psi_z ||^2=n||Ap||^2\leq (2nT^2+56n^3T\sqrt{T})(d+2\sqrt{d\ln(1/\delta)}+2\ln(1/\delta)).
\end{equation}
We are now ready to prove our main result.
\begin{proof} [\textbf{Proof of theorem~\ref{thm:main}}]
  For $n\geq n_{0,\delta}$ and $T\geq \max\{T_{0,\delta},T_{1,\delta}\}$ from lemma~ \ref{lem:inv_sigh},~\ref{lem:inv_sighz}, both $\Sighz$, and $\Sigh$ are invertible with probability atleast $1-\delta, 1-4\delta$ respectively. Conditioned on the invertibility of $\Sighz,\Sigma$ we get from lemmas~\ref{lem:1}-\ref{lem:3}, $||\Sigi\Sigh^{1/2}||^2\leq 3/2$ and $||\Sig^{1/2}\Sighzi\Sig^{1/2}||^2\leq400/n^2T^2$, and $||\Sigh^{-1/2}\psi_z||^2\leq (2nT^2+56n^3T^{3/2})(d+2\sqrt{d\ln(1/\delta)+2\ln(1/\delta)})$ with probability atleast $1-\delta,1-3\delta,1-\delta$ respectively. Using lemma~\ref{lem:decompose} and the union bound to add up all the failure probabilities we get the desired result.
\end{proof}
\section{Related Work}
A variety of pool based AL algorithms have been proposed in the literature employing various query strategies. However, none of them use unbiased estimates of the risk. One of the simplest strategy for AL is uncertainty sampling, where the active learner queries the point whose label it is most uncertain about. This strategy has been popularl in text classification~\citep{lewis1994sequential}, and information extraction~\citep{settles2008analysis}. Usually the uncertainty in the label is calculated using certain information-theoretic criteria such as entropy, or variance of the label distribution. While uncertainty sampling has mostly been used in a probabilistic setting, AL algorithms which learn non-probabilistic classifiers using uncertainty sampling have also been proposed. Tong et al.~\citeyearpar{tong2001support} proposed an algorithm in this framework where they query the point closest to the current svm hyperplane. Seung et al.~\citeyearpar{seung1992query} introduced the query-by-committee (QBC) framework where a committee of potential models, which all agree on the currently labeled data is maintained and, the point where most committee members disagree is considered for querying. In order to design a committee in the QBC framework, algorithms such as query-by-boosting, and query-by-bagging in the discriminative setting~\citep{Abe1998query}, sampling from a Dirichlet distribution over model parameters in the generative setting~\citep{mccallumzy1998employing} have been proposed. Other frameworks include querying the point, which causes the maximum expected reduction in error~\citep{zhu2003combining,guo2007optimistic}, variance reducing query strategies such as the ones based on optimal design~\citep{flaherty2005robustdesign,zhang2000value}. A very thorough literature survey of different active learning algorithms has been done by Settles~\citeyearpar{settlestr09}. AL algorithms that are consistent and have provable label complexity have been proposed for the agnostic setting for the 0-1 loss in recent years~\citep{dasgupta2007general,beygelzimer2009importance}. The IWAL framework introduced in Beygelzimer et al.~\citeyearpar{beygelzimer2009importance} was the first AL algorithm with guarantees for general loss functions. However the authors were unable to provide non-trivial label complexity guarantees for the hinge loss, and the squared loss. 

UPAL at least for squared losses can be seen as using a QBC based querying strategy where the committee is the entire hypothesis space, and the disagreement among the committee members is calculated using an exponential weighting scheme. However unlike previously proposed committees our committee is an infinite set, and the choice of the point to be queried is randomized.
\section{Experimental results}\label{sec:expts}
We implemented UPAL, along with the standard passive learning (PL) algorithm, and a variant of UPAL called RAL (in short for random active learning), all using logistic loss, in matlab. The choice of logistic loss was motivated by the fact that BMAL was designed for logistic loss. Our matlab codes were vectorized to the maximum possible extent so as to be as efficient as possible. RAL is similar to UPAL, but in each round samples a point uniformly at random from the currently unqueried pool. However it does not use importance weights to calculate an estimate of the risk of the classifier. The purpose of implementing RAL was to demonstrate the potential effect of using unbiased estimators, and to check if the strategy of randomly querying points helps in active learning.

We also implemented a batch mode active learning algorithm introduced by Hoi et al.~\citeyearpar{hoi2006batch} which, we shall call as BMAL. Hoi et al. in their paper showed superior empirical performance of BMAL over other competing pool based active learning algorithms, and this is the primary motivation for choosing BMAL as a competitor pool AL algorithm in this paper. BMAL like UPAL also proceeds in rounds and in each iteration selects $k$ examples by minimizing the Fisher information ratio between the current unqueried pool and the queried pool. However a point once queried by BMAL is never requeried. In order to tackle the high computational complexity of optimally choosing a set of $k$ points in each round, the authors suggested a monotonic submodular approximation to the original Fisher ratio objective, which is then optimized by a greedy algorithm.  
At the start of round $t+1$ when, BMAL has already queried $t$ points in the previous rounds, in order to decide which point to query next, BMAL has to calculate for each potential new query a dot product with all the remaining unqueried points. Such a calculation when done for all possible potential new queries takes $O(n^2t)$ time. Hence if our budget is $B$, then the total computational complexity of BMAL is $O(n^2B^2)$. Note that this calculation does not take into account the  complexity of solving an optimization problem in each round after having queried a point. In order to further reduce the computational complexity of BMAL in each round we further restrict our search, for the next query, to a small subsample of the current set of unqueried points. We set the value of $p_{\text{min}}$ in step 3 of algorithm 1 to $\frac{1}{nt}$.
In order to avoid numerical problems we implemented a regularized version of UPAL where the term $\lambda||w||^2$ was added to the optimization problem shown in step 11 of Algorithm 1. The value of $\lambda$ is allowed to change as per the current importance weight of the pool. The optimal value of $C$ in VW~\footnote{The parameters initial\_t, $l$ were set to a default value of 10 for all of our experiments.} was chosen via a 5 fold cross-validation, and by eyeballing for the value of $C$ that gave the best cost-accuracy trade-off. We ran all our experiments on the MNIST  dataset(3 Vs 5)~\footnote{The dataset can be obtained from \url{http://cs.nyu.edu/~roweis/data.html}. We first performed PCA to reduce the dimensions to 25 from 784.}, and datasets from UCI repository namely Statlog, Abalone, Whitewine. Figure~\ref{fig:expt_results} shows the performance of all the algorithms on the first 300 queried points.
\begin{figure*}[tbph]
  \centering
  \subfigure[MNIST (3 vs 5)]{\includegraphics[scale=0.30]{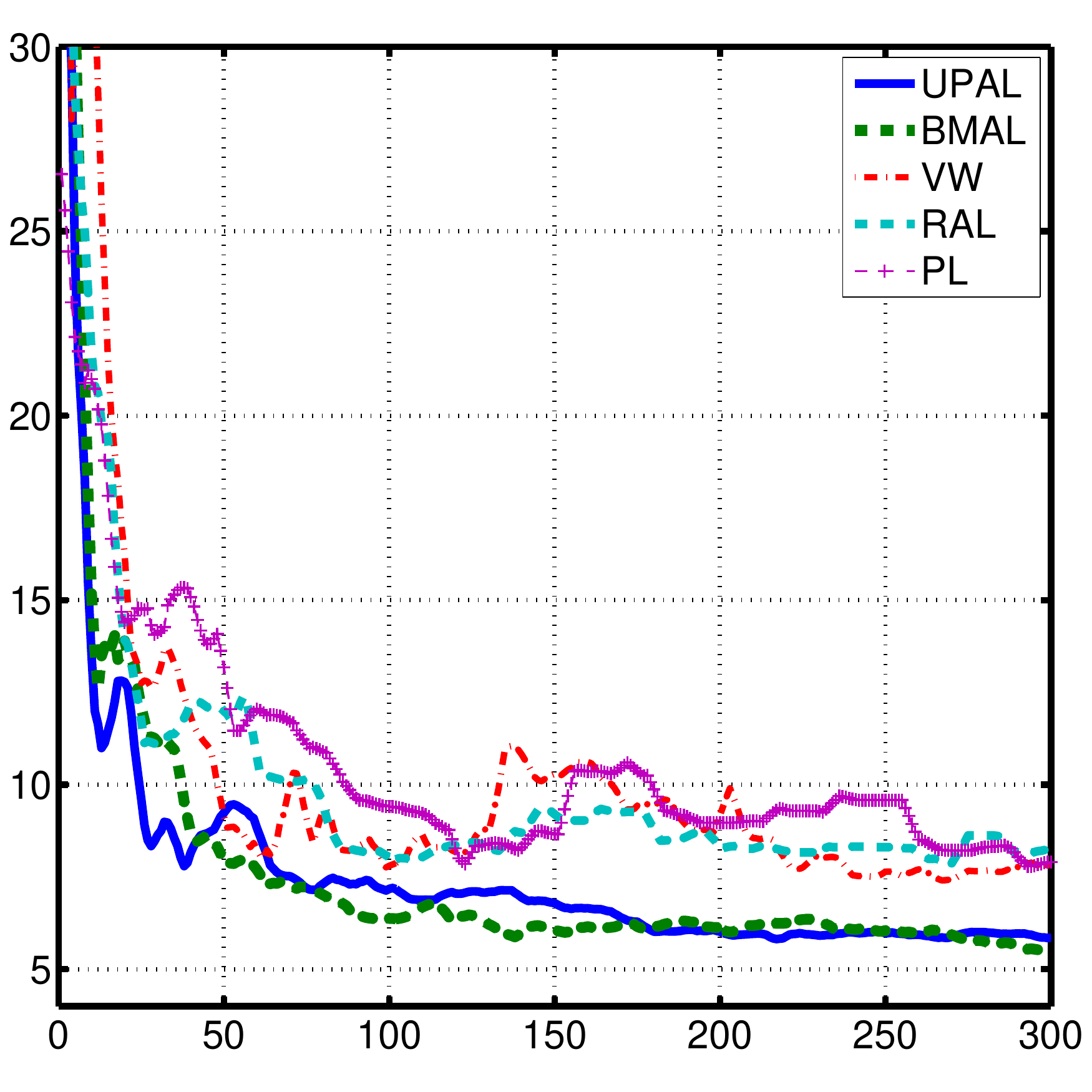}}
  \subfigure[Statlog]{\includegraphics[scale=0.30]{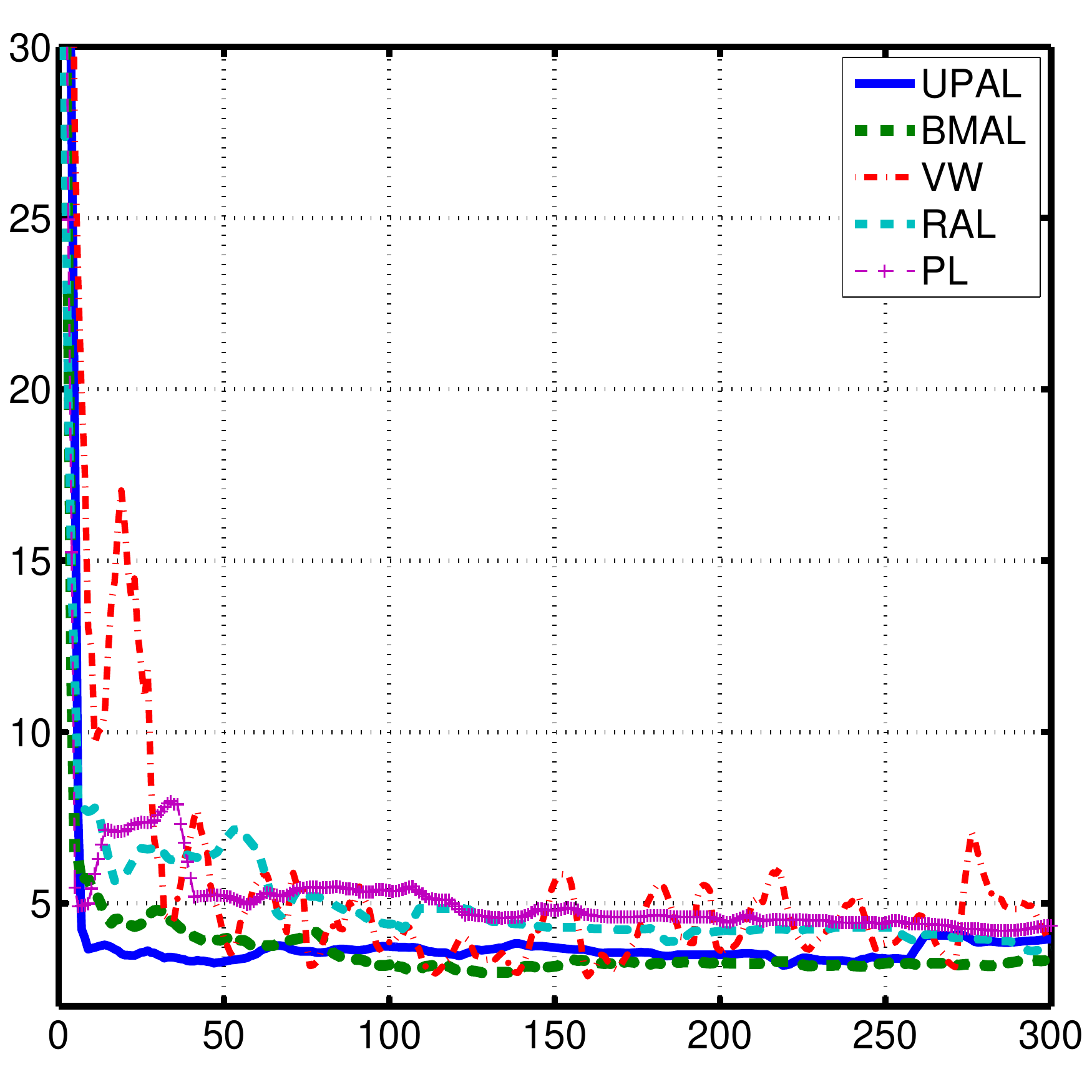}}
 \subfigure[Abalone]{\includegraphics[scale=0.30]{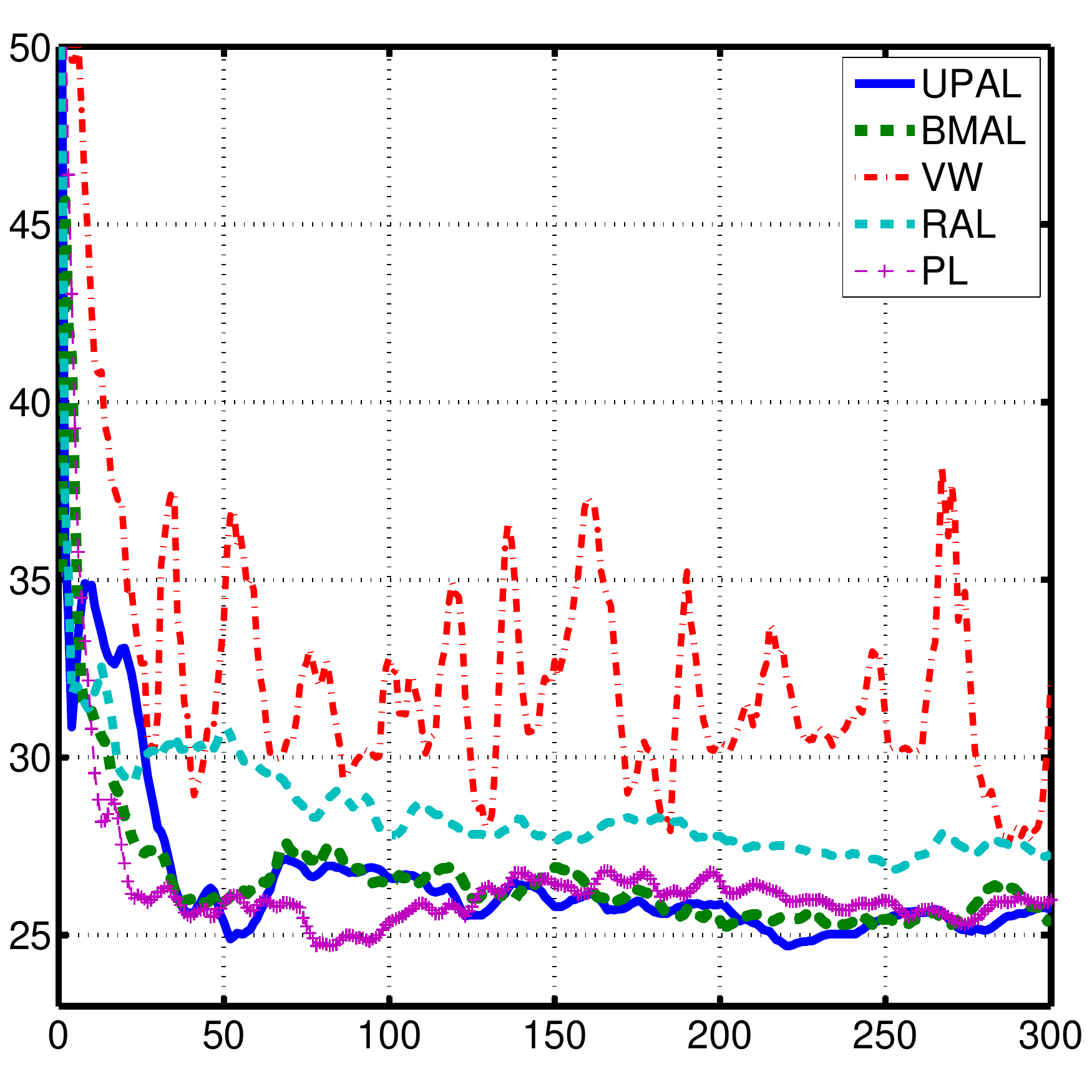}}
\end{figure*}
\begin{figure}
  \centering
  \subfigure[Whitewine]{\includegraphics[scale=0.35]{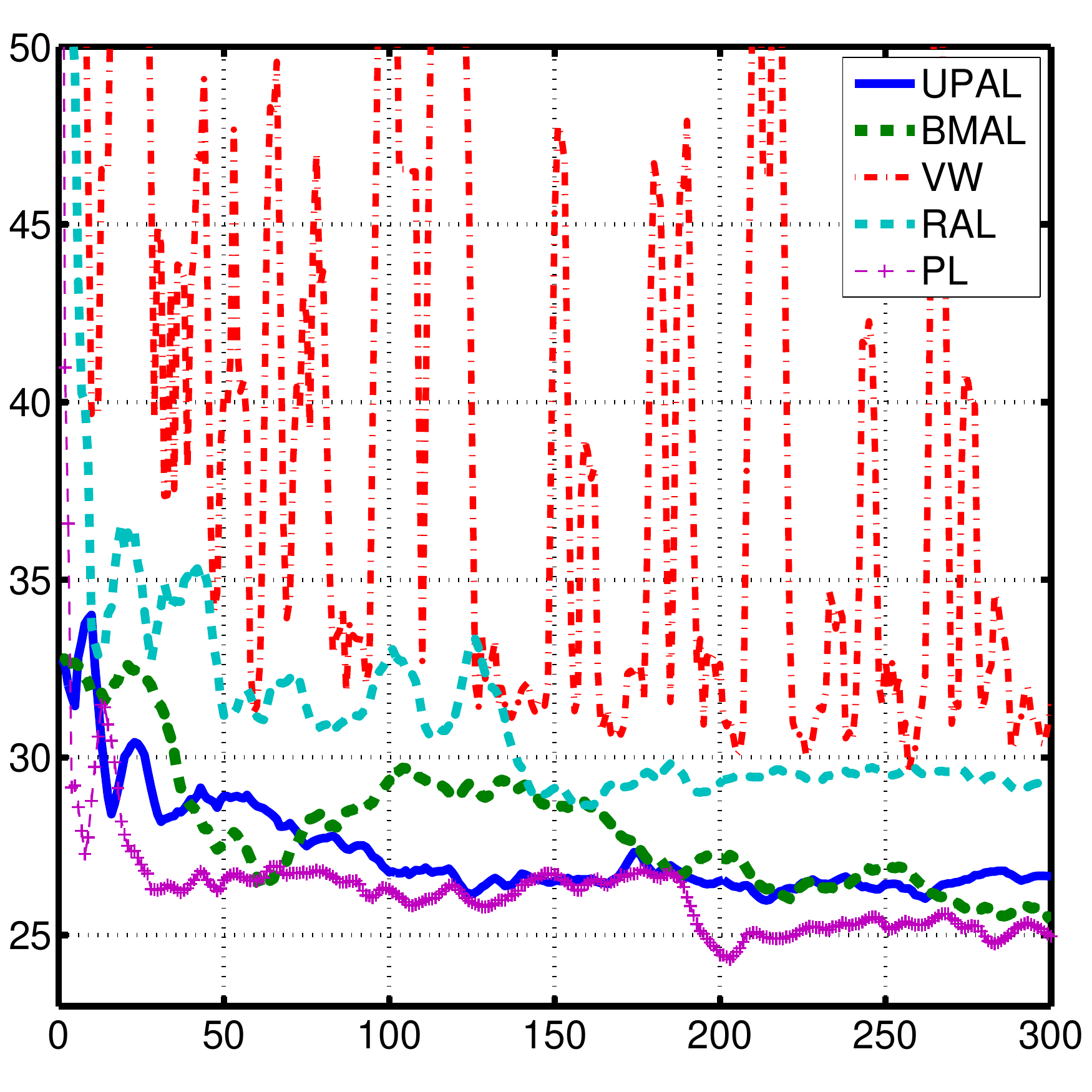}}
 \caption{\label{fig:expt_results}Empirical performance of passive and active learning algorithms.The x-axis represents the number of points queried, and the y-axis represents the test error of the classifier. The subsample size for approximate BMAL implementation was fixed at 300.}
\end{figure}
  \begin{table}[hbpt]
    \centering
    \begin{tabular}{|c|c|c|c|c|}
      \hline
     Sample size& \multicolumn{2}{|c|} {UPAL}&\multicolumn{2}{|c|}{BMAL}\\
     \hline
     &Time&Error&Time&Error\\
     \hline
     1200&65&7.27&60&5.67\\
     \hline
     2400&100&6.25&152&6.05\\
     \hline
     4800&159&6.83&295&6.25\\
     \hline
     10000&478&5.85&643.17&5.85\\
     \hline
    \end{tabular}
    \caption{\label{tab:fixed_B}\footnotesize{Comparison of UPAL and BMAL on MNIST data-set of varying training sizes, and with the budget being fixed at 300. The error rate is in percentage, and the time is in seconds.}}
\end{table}
  \begin{table}[hbpt]
    \centering
    \begin{tabular}{|c|c|c|c|c|c|}
      \hline
      Budget& \multicolumn{2}{|c|} {UPAL}&\multicolumn{2}{|c|}{BMAL}&Speedup\\
      \hline
      &Time&Error&Time&Error&\\
      \hline
      500&859&5.79&1973&5.33&2.3\\
      \hline
      1000&1919&6.43&7505&5.70&3.9\\
      \hline
      2000&4676&5.82&32186&5.59&6.9\\
      \hline
    \end{tabular}
    \caption{\label{tab:fixed_n}\footnotesize{Comparison of UPAL on the entire MNIST dataset for varying budget size. All the times are in seconds unless stated, and error rates in percentage.}}
  \end{table}
On the MNIST dataset, on an average, the performance of BMAL is very similar to UPAL, and there is a noticeable gap in the performance of BMAL and UPAL over PL, VW and RAL. Similar results were also seen in the case of Statlog dataset, though towards the end the performance of UPAL slightly worsens when compared to BMAL. However UPAL is still better than  PL, VW, and RAL.

Active learning is not always helpful and the success story of AL depends on the match between the marginal distribution and the hypothesis class. This is clearly reflected in Abalone where the performance of PL is better than UPAL atleast in the initial stages and is never significantly worse. UPAL is uniformly better than BMAL, though the difference in error rates is not significant.  However the performance of RAL, VW are significantly worse. Similar results were also seen in the case of Whitewine dataset, where PL outperforms all AL algorithms. UPAL is better than BMAL most of the times. Even here one can witness a huge gap in the performance of VW and RAL over PL, BMAL and UPAL.

One can conclude that VW though is computationally efficient has higher error rate for the same number of queries. The uniformly poor performance of RAL signifies that querying uniformly at random does not help. On the whole UPAL and BMAL perform equally well, and we show via our next set of experiments that UPAL has significantly better scalability, especially when one has a relatively large budget $B$.

\subsection{Scalability results}
Each round of UPAL takes $O(n)$ plus the time to solve the optimization problem shown in step 11 in Algorithm 1. A similar optimization problem is also solved in the BMAL problem. If the cost of solving this optimization problem in step $t$ is $c_{opt,t}$, then the complexity of UPAL is $O(nT+\sum_{t=1}^T c_{opt,t})$. While BMAL takes $O(n^2B^2+\sum_{t=1}^Tc'_{t,opt})$ where $c'_{t,opt}$ is the complexity of solving the optimization problem in BMAL in round $t$. For the approximate implementation of BMAL that we described if the subsample size is $|S|$, then the complexity is $O(|S|^2B^2+\sum_{t=1}^Tc'_{t,opt})$. 

In our first set of experiments we fix the budget $B$ to 300, and calculate the test error  and the combined training and testing time of both BMAL and UPAL for varying sizes of the training set. All the experiments were performed on the MNIST dataset. Table~\ref{tab:fixed_B} shows that with increasing sample size UPAL tends to be more efficient than BMAL, though the gain in speed that we observed was at most a factor of 1.8. 

In the second set of scalability experiments we fixed the training set size to 10000, and studied the effect of increasing budget. We found out that with increasing budget size the speedup of UPAL over BMAL increases. In particular when the \textit{budget was 2000, UPAL is arpproximately 7 times faster than BMAL.} All our experiments were run on a dual core machine with 3 GB memory.
\section{Conclusions and Discussion}
In this paper we proposed the first unbiased pool based active learning algorithm, and showed its good empirical performance and its ability to scale both with higher budget constraints and large dataset sizes. Theoretically we proved that when the true hypothesis is a linear hypothesis, we are able to recover it with high probability. In our view an important extension of this work would be to establish tighter bounds on the excess risk. It should be possible to provide upper bounds on the excess risk in expectation which are much sharper than our current high probability bounds. Another theoretically interesting question is to calculate how many unique queries are made after $T$ rounds of UPAL. This problem is similar to calculating the number of non-empty bins in the balls-and-bins model commonly used in the field of randomized algorithms~\cite{motwani1995ra}, when there are $n$ bins and $T$ balls, with the different points in the pool being the bins, and the process of throwing a ball in each round being equivalent to querying a point in each round. However since each round is, unlike standard balls-and-bins, dependent on the previous round we expect the analysis to be more involved than a standard balls-and-bins analysis.
\bibliographystyle{plainnat}
\bibliography{upal_arxiv.bib}

\begin{thebibliography}{31}
\providecommand{\natexlab}[1]{#1}
\providecommand{\url}[1]{\texttt{#1}}
\expandafter\ifx\csname urlstyle\endcsname\relax
  \providecommand{\doi}[1]{doi: #1}\else
  \providecommand{\doi}{doi: \begingroup \urlstyle{rm}\Url}\fi

\bibitem[Abe and Mamitsuka(1998)]{Abe1998query}
N.~Abe and H.~Mamitsuka.
\newblock Query learning strategies using boosting and bagging.
\newblock In \emph{ICML}, 1998.

\bibitem[Baum and Lang(1992)]{baum1992query}
E.B. Baum and K.~Lang.
\newblock Query learning can work poorly when a human oracle is used.
\newblock In \emph{IJCNN}, 1992.

\bibitem[Beygelzimer et~al.(2009)Beygelzimer, Dasgupta, and
  Langford]{beygelzimer2009importance}
A.~Beygelzimer, S.~Dasgupta, and J.~Langford.
\newblock {Importance weighted active learning}.
\newblock In \emph{ICML}, 2009.

\bibitem[Cesa-Bianchi and Lugosi(2006)]{cesa2006prediction}
N.~Cesa-Bianchi and G.~Lugosi.
\newblock \emph{Prediction, learning, and games}.
\newblock Cambridge Univ Press, 2006.

\bibitem[Chu et~al.(2011)Chu, Zinkevich, Li, Thomas, and
  Tseng]{chu2011unbiased}
W.~Chu, M.~Zinkevich, L.~Li, A.~Thomas, and B.~Tseng.
\newblock Unbiased online active learning in data streams.
\newblock In \emph{SIGKDD}, 2011.

\bibitem[Cohn et~al.(1994)Cohn, Atlas, and Ladner]{cohn1994improving}
D.~Cohn, L.~Atlas, and R.~Ladner.
\newblock {Improving generalization with active learning}.
\newblock \emph{Machine Learning}, 15\penalty0 (2), 1994.

\bibitem[Dasgupta et~al.(2007)Dasgupta, Hsu, and
  Monteleoni]{dasgupta2007general}
S.~Dasgupta, D.~Hsu, and C.~Monteleoni.
\newblock {A general agnostic active learning algorithm}.
\newblock \emph{NIPS}, 2007.

\bibitem[Flaherty et~al.(2005)Flaherty, Jordan, and
  Arkin]{flaherty2005robustdesign}
Patrick Flaherty, Michael~I. Jordan, and Adam~P. Arkin.
\newblock Robust design of biological experiments.
\newblock In \emph{Neural Information Processing Systems}, 2005.

\bibitem[Guo and Greiner(2007)]{guo2007optimistic}
Y.~Guo and R.~Greiner.
\newblock {Optimistic active learning using mutual information}.
\newblock In \emph{IJCAI}, 2007.

\bibitem[Hoi et~al.(2006)Hoi, Jin, Zhu, and Lyu]{hoi2006batch}
S.C.H. Hoi, R.~Jin, J.~Zhu, and M.R. Lyu.
\newblock {Batch mode active learning and its application to medical image
  classification}.
\newblock In \emph{ICML}, 2006.

\bibitem[Horn and Johnson(1990)]{horn90matrix}
R.A. Horn and C.R. Johnson.
\newblock \emph{{Matrix analysis}}.
\newblock Cambridge Univ Press, 1990.

\bibitem[Hsu et~al.(2011{\natexlab{a}})Hsu, Kakade, and Zhang]{hsu2011analysis}
D.~Hsu, S.M. Kakade, and T.~Zhang.
\newblock An analysis of random design linear regression.
\newblock \emph{Arxiv preprint arXiv:1106.2363}, 2011{\natexlab{a}}.

\bibitem[Hsu et~al.(2011{\natexlab{b}})Hsu, Kakade, and
  Zhang]{hsu2011dimension}
D.~Hsu, S.M. Kakade, and T.~Zhang.
\newblock Dimension-free tail inequalities for sums of random matrices.
\newblock \emph{Arxiv preprint arXiv:1104.1672}, 2011{\natexlab{b}}.

\bibitem[Langford et~al.(2011)Langford, Li, Strehl, Hsu, Karampatziakis, and
  Hoffman]{langford2010vowpal}
J.~Langford, L.~Li, A.~Strehl, D.~Hsu, N.~Karampatziakis, and M.~Hoffman.
\newblock Vowpal wabbit, 2011.

\bibitem[Lewis and Gale(1994)]{lewis1994sequential}
D.D. Lewis and W.A. Gale.
\newblock {A sequential algorithm for training text classifiers}.
\newblock In \emph{SIGIR}, 1994.

\bibitem[Litvak et~al.(2005)Litvak, Pajor, Rudelson, and
  Tomczak-Jaegermann]{litvak2005smallest}
AE~Litvak, A.~Pajor, M.~Rudelson, and N.~Tomczak-Jaegermann.
\newblock Smallest singular value of random matrices and geometry of random
  polytopes.
\newblock \emph{Advances in Mathematics}, 195\penalty0 (2):\penalty0 491--523,
  2005.

\bibitem[McCallum and Nigam(1998)]{mccallumzy1998employing}
A.K. McCallum and K.~Nigam.
\newblock {Employing EM and pool-based active learning for text
  classification}.
\newblock In \emph{ICML}, 1998.

\bibitem[Motwani and Raghavan(1995)]{motwani1995ra}
Rajeev Motwani and Prabhakar Raghavan.
\newblock \emph{Randomized Algorithms}.
\newblock Cambridge University Press, 1st edition, August 1995.

\bibitem[Quinonero et~al.(2008)Quinonero, Sugiama, Schwaighofer, and
  Lawrence]{quinonero2008dataset}
J.~Quinonero, M.~Sugiama, A.~Schwaighofer, and N.D. Lawrence.
\newblock Dataset shift in machine learning, 2008.

\bibitem[Rahimi and Recht(2007)]{rahimi2007random}
Ali Rahimi and Benjamin Recht.
\newblock Random features for large-scale kernel machines.
\newblock In \emph{Neural Information Processing Systems}, 2007.

\bibitem[Rokhlin and Tygert(2008)]{rokhlin2008fast}
V.~Rokhlin and M.~Tygert.
\newblock A fast randomized algorithm for overdetermined linear least-squares
  regression.
\newblock \emph{Proceedings of the National Academy of Sciences}, 105\penalty0
  (36):\penalty0 13212, 2008.

\bibitem[Settles and Craven(2008)]{settles2008analysis}
B.~Settles and M.~Craven.
\newblock An analysis of active learning strategies for sequence labeling
  tasks.
\newblock In \emph{EMNLP}, 2008.

\bibitem[Settles(2009)]{settlestr09}
Burr Settles.
\newblock Active learning literature survey.
\newblock Computer Sciences Technical Report 1648, University of
  Wisconsin--Madison, 2009.

\bibitem[Seung et~al.(1992)Seung, Opper, and Sompolinsky]{seung1992query}
H.S. Seung, M.~Opper, and H.~Sompolinsky.
\newblock Query by committee.
\newblock In \emph{COLT}, pages 287--294. ACM, 1992.

\bibitem[Shamir(2011)]{shamir2011variant}
O.~Shamir.
\newblock A variant of azuma's inequality for martingales with subgaussian
  tail.
\newblock \emph{Arxiv preprint arXiv:1110.2392}, 2011.

\bibitem[Tong and Chang(2001)]{tong2001support}
S.~Tong and E.~Chang.
\newblock Support vector machine active learning for image retrieval.
\newblock In \emph{Proceedings of the ninth ACM international conference on
  Multimedia}, 2001.

\bibitem[Tropp(2010)]{tropp2010user}
J.A. Tropp.
\newblock User-friendly tail bounds for sums of random matrices.
\newblock \emph{Arxiv preprint arXiv:1004.4389}, 2010.

\bibitem[van~de Geer(2000)]{vandegeer2000empirical}
Sara van~de Geer.
\newblock Empirical processes in m-estimation.
\newblock 2000.

\bibitem[Zhang(2004)]{zhang2004statistical}
T.~Zhang.
\newblock {Statistical behavior and consistency of classification methods based
  on convex risk minimization}.
\newblock \emph{Annals of Statistics}, 32\penalty0 (1), 2004.

\bibitem[Zhang and Oles(2000)]{zhang2000value}
T.~Zhang and F.~Oles.
\newblock The value of unlabeled data for classification problems.
\newblock In \emph{ICML}, 2000.

\bibitem[Zhu et~al.(2003)Zhu, Lafferty, and Ghahramani]{zhu2003combining}
Xiaojin Zhu, John Lafferty, and Zoubin Ghahramani.
\newblock Combining active learning and semi-supervised learning using gaussian
  fields and harmonic functions.
\newblock In \emph{ICML}, 2003.

\end{thebibliography}
\appendix
\section{Some results from random matrix theory}
\begin{thm} \label{thm:quadratic}(Quadratic forms of subgaussian random vectors~\citep{litvak2005smallest,hsu2011analysis}) Let $A\in \bbR^{m\times n}$ be a matrix, and $H\defeq AA^T$, and $r=(r_1,\ldots,r_n)$ be a random vector such that for some $\sigma\geq0$, 
  \begin{equation*}
    \bbE[\exp(\alpha^Tr)]\leq \exp\left(\frac{||\alpha||^2\sigma^2}{2}\right)
  \end{equation*}
  for all $\alpha\in \bbR^n$ almost surely. For all $\delta\in(0,1)$,
  \begin{equation*}
    \bbP~\left[||Ar||^2>\sigma^2\tr(H)+2\sigma^2\sqrt{\tr(H^2)}\ln(1/\delta)+2\sigma^2||H||\ln(1/\delta)\right]\leq \delta.
  \end{equation*}
\end{thm}
The above theorem was first proved without explicit constants by Litvak  et al.~\citep{litvak2005smallest} Hsu et al~\citep{hsu2011analysis} established a version of the above theorem with explicit constants.  
\begin{thm}\label{thm:litvak}(Eigenvalue bounds of a sum of rank-1 matrices) Let $r_1,\ldots r_n$ be random vectors in $\bbR^d$ such that, for some $\gamma>0$, 
  \begin{align*}
    \bbE[r_ir_i^T|r_1,\ldots,r_{i-1}]&=I\\
   \bbE[\exp(\alpha^Tr_i)|r_1,\ldots,r_{i-1}]&\leq \exp(||\alpha||^2\gamma/2) ~\forall \alpha \in \bbR^d.
  \end{align*}
  For all $\delta \in (0,1)$,
  \begin{equation*}
    \bbP\left[\lambdamax\left(\frac{1}{n}\sum_{i=1}^n r_ir_i^T\right)>1+2\epsilon_{\delta,n} \vee  \lambdamin\left(\frac{1}{n}\sum_{i=1}^n r_ir_i^T\right)<1-2\epsilon_{\delta,n}\right]\leq \delta,
  \end{equation*}
  where
  \begin{equation*}
    \epsilon_{\delta,n}=\gamma\left(\sqrt{\frac{32(d~\ln(5)+\ln(2/\delta))}{n}}+\frac{2(d\ln(5)+\ln(2/\delta))}{n}\right).
  \end{equation*}
\end{thm}
We shall use the above theorem in Lemma ~\ref{lem:inv_sigh}, and lemma~\ref{lem:J}.
\begin{thm} \label{thm:mat_bern}(Matrix Bernstein bound) Let $X_1\ldots,X_n$ be symmetric valued random matrices. Suppose there exist $\bar{b},\bar{\sigma}$ such that for all $i=1,\ldots,n$
\begin{align*}
  \bbE_i[X_i]&=0\\
  \lambdamax(X_i)&\leq \bar{b}\\
  \lambdamax\left(\frac{1}{n}\sum_{i=1}^n  \bbE_{i}[X_i^2]\right)&\leq \bar{\sigma}^2.
\end{align*}
almost surely, then
\begin{align}
\bbP\left[\lambdamax\left(\frac{1}{n}\sum_{i=1}^n X_i\right)> \sqrt{\frac{2\bar{\sigma}^2\ln(d/\delta)}{n}}+\frac{\bar{b}\ln(d/\delta)}{3n}\right]\leq \delta.
\end{align}
A dimension free version of the above inequality was proved in Hsu et al~\citep{hsu2011dimension}. Such dimension free inequalities are especially useful in infinite dimension spaces. Since we are working in finite dimension spaces, we shall stick to the non-dimension free version.
\end{thm}
\begin{thm}~\citep{shamir2011variant}\label{lem:modazuma}
  Let $(Z_1,\cF_1),\ldots,(Z_T,\cF_T)$ be a martingale difference sequence, and suppose there are constants $b\geq 1,c_t>0$ such that for any $t$ and any $a>0$,
\begin{equation*}
  \max\{\bbP[Z_t\geq a|\cF_{t-1}],\bbP[Z_t\leq -a|\cF_{t-1}]\}\leq b\exp(-c_ta^2).
\end{equation*}
Then for any $\delta>0$, with probability atleast $1-\delta$ we have
\begin{equation*}
  \frac{1}{T}\sum_{t=1}^T Z_t\leq \sqrt{\frac{28b\ln(1/\delta)}{\sum_{t=1}^Tc_t}}.
\end{equation*}
\end{thm}
The above result was first proved by Shamir~\citep{shamir2011variant}. Shamir proved the result for the case when $c_1=\ldots=c_{T}$. Essentially one can use the same proof with obvious changes to get the above result.
\begin{lemma}[Hoeffding's lemma]~\citep[see][page 359]{cesa2006prediction}
  Let $X$ be a random variable with $a\leq X\leq b$. Then for any $s\in\bbR$
  \begin{equation}
    \bbE[\exp(sX)]\leq \exp\left(s\bbE[X]+\frac{s^2(b-a)^2}{8}\right).
  \end{equation}
\end{lemma}
\begin{thm} Let $A, B$ be  positive semidefinite matrices. Then 
\begin{equation*}
 \lambdamax(A)+\lambdamin(B)\leq \lambdamax(A+B)\leq \lambdamax(A)+\lambdamax(B).
\end{equation*}
The above inequalities are called as Weyl's inequalities~\citep[see][chap. 3]{horn90matrix}
\end{thm}
\end{document}